\documentclass[preprint,3p,12pt]{elsarticle}
\makeatletter
\def\ps@pprintTitle{%
 \let\@oddhead\@empty
 \let\@evenhead\@empty
 \def\@oddfoot{\centerline{\thepage}}%
 \let\@evenfoot\@oddfoot}
\makeatother
\usepackage{soul}
\usepackage{graphics}
\usepackage{psfig}
\usepackage{amsmath}
\usepackage{rotating}
\usepackage[absolute]{textpos}
\usepackage{color}
\usepackage{lineno}
\usepackage[utf8]{inputenc}
\usepackage{amsmath}
\usepackage{amsfonts}
\usepackage{amssymb}
\usepackage{amsthm}
\usepackage{mathrsfs}
\usepackage{latexsym}
\usepackage{multirow}
\usepackage{rotating}
\usepackage{caption}
\usepackage{graphicx}
\usepackage{float}
\usepackage{array}
\usepackage{subfigure}
\usepackage{tabularx}
\usepackage{graphicx,psfig}
\usepackage{footnote}
\usepackage[titletoc,toc,title]{appendix}
\usepackage{colortbl}
\usepackage[dvipsnames]{xcolor}
\usepackage{hyperref}
\usepackage{natbib}
\usepackage{amssymb}
\usepackage{soul}
\usepackage{algorithmic}
\usepackage{algorithm}
\usepackage{epsfig}
\usepackage{relsize}
\usepackage{color}
\usepackage{mathtools}
\usepackage{multirow}
\usepackage[utf8]{inputenc}
\usepackage[english]{babel}

\newtheorem{theorem}{Theorem}

\newtheorem*{remark}{Remark}

\begin{document}

\begin{frontmatter}

\title{Superensemble Classifier for Improving Predictions in Imbalanced Datasets}
\author{Tanujit Chakraborty\footnote[1]{\textit{Corresponding author}:
Tanujit Chakraborty (tanujit\_r@isical.ac.in)}
\href{https://orcid.org/0000-0002-3479-2187}{\includegraphics[scale=1]{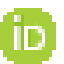}}
and Ashis Kumar Chakraborty\textsuperscript{2}\\
    {\scriptsize \textsuperscript{1 and 2} SQC and OR Unit, Indian Statistical Institute, 203, B. T. Road, Kolkata - 700108, India}\\}
%\doublespacing

\begin{abstract}
Learning from an imbalanced dataset is a tricky proposition. Because
these datasets are biased towards one class, most existing
classifiers tend not to perform well on minority class examples.
Conventional classifiers usually aim to optimize the overall
accuracy without considering the relative distribution of each
class. This article presents a superensemble classifier, to tackle
and improve predictions in imbalanced classification problems, that
maps Hellinger distance decision trees (HDDT) into radial basis
function network (RBFN) framework. Regularity conditions for
universal consistency and the idea of parameter optimization of the
proposed model are provided. The proposed distribution-free model
can be applied for feature selection cum imbalanced classification
problems. We have also provided enough numerical evidence using
various real-life data sets to assess the performance of the
proposed model. Its effectiveness and competitiveness with respect
to different state-of-the-art models are shown.
\end{abstract}


\end{frontmatter}

\section{Introduction}
Distribution-free models are specially used in the fields of
statistics and data mining for more than fifty years now, mainly for
their robustness and ability to deal with complex data structures
\citep{romeu2006operations}. However, traditional classifiers
usually make a simple assumption that the classes to be
distinguished should have a comparable number of instances
\citep{kuncheva2004combining}. This assumption doesn't often hold in
many real-world classification problems. Real-world data sets are
usually skewed, in which many of the cases belong to a larger class,
and fewer cases belong to a smaller, yet usually more interesting
class. These are also the cases where the cost of misclassifying
minority examples is much higher concerning the seriousness of the
problem we are dealing with \citep{rastgoo2016tackling}. Due to
higher weightage is given to the majority class, these systems tend
to misclassify the minority class examples as the majority, and lead
to a high false negative rate. For example, consider a binary
classification problem with the class distribution of $90:10$. In
this case, a straightforward method of guessing all instances to be
positive class would achieve an accuracy of $90\%$.

To deal with imbalanced datasets, there are many approaches
developed in the literature. One way to deal with the imbalanced
data problems is to modify the class distributions in the training
data by applying sampling techniques to the dataset. Sampling
techniques either oversamples the minority class to match the size
of the majority class \citep{guo2004learning} or undersamples the
majority class to match the size of the minority class
\citep{chen2004using}. Synthetic minority oversampling technique
(SMOTE) is among the most popular methods that oversample the
minority class by generating artificially interpolated data
\citep{chawla2002smote}. Hybrid sampling approaches, viz. SMOTE with
data cleaning methods (for example, Tomek links (TL) and edited
nearest neighbor (ENN) rule) not only balances the data but also
removes noisy instances lying on the wrong side of the decision
boundaries \citep{batista2004study}. Sometimes different
combinations of undersampling, oversampling and ensemble learning
techniques are used to tackle the curse of imbalanced datasets
\citep{lemaitre2017imbalanced}. But these approaches have apparent
deficiencies like undersampling majority instances may lose
potential useful information of the data set and oversampling
increases the size of the training data set, which may increase
computational cost. To overcome this problem, ``imbalanced
data-oriented" algorithms are designed which can handle class
imbalanced without any modification in class distribution. HDDT
\citep{cieslak2008learning} uses Hellinger distance (HD) as decision
tree splitting criterion and it is insensitive towards the skewness
of the class distribution \citep{cieslak2012hellinger}. Some other
pieces of literature are HD based random forest (HDRF)
\citep{su2015improving} and class confidence proportion decision
tree (CCPDT), robust decision tree-based algorithms which can also
handle original imbalanced datasets \citep{liu2010robust}. Though
HDDTs are robust, skew-sensitive and mitigate the need for sampling,
they are high variance estimator and a greedy algorithm.

To mitigate these problems of HDDT suffering from sticking to local
minima and overfitting the data set, an ensemble learning approach
is adopted in this paper. Important prerequisites for building a
``good" ensemble classifier is to choose the base classifier to be
simple and as accurate as possible and distinct from the other
classifier(s) used \citep{kuncheva2004combining}. Two such widely
used models for ensemble learning are decision tree (DT) and
artificial neural networks (ANN). Extensive works are done earlier
on the mapping of tree-based models to ANN in the previous
literature \citep{sethi1990comparison, sakar1993growing,
kubat1998decision, chen2006feature, chakraborty2018novel}. However,
training multilayer perceptrons (MLP) usually takes longer time and
finding the number of nodes in the hidden layer of MLP is a
challenging task whereas RBFN has the advantages of having only one
hidden layer, faster convergence, easy interpretability and
universal consistency \citep{park1991universal}. But RBFN also
assumes having equal class distribution during implementation in
classification problems. Motivated by the above discussion, we
propose in the present paper a novel superensemble classifier based
on HDDT and RBFN. Harnessing the ensemble formulation, we try to
exploit the strengths of HDDT and RBFN models to overcome their
drawbacks. The approach is first developed in theoretical details
and latter different training schemes are experimentally evaluated
on various small and medium-sized real world imbalanced data sets
having high dimensional feature spaces. The proposed superensemble
model has the advantages of significant accuracy, very less number
of tuning parameters and ability to handle small or medium-sized
datasets. Another major advantage of the proposed algorithm is its
interpretability as compared to more ``black-box-like" complex
models. Our proposed distribution-free superensemble classifier is
found to be universally consistent and an idea on parameter
optimization of the model is also proposed in this paper.

This paper is organized as follows. Section 2 illustrates the major
problems when the dataset is imbalanced in nature. Section 3
outlines the HDDT, RBFN algorithm and the details of the proposed
hybrid approach. Section 4 presents the theoretical properties of
the approach. Section 5 is devoted to computational experiments and
comparisons. Conclusions and discussions of the paper are given in
Section 6.

\section{Imbalanced Classification Problem}

Let us first investigate the effect of class imbalance on the
performance metrics and DT. It is essential to see how decision
boundaries created by DT get affected by imbalance ratio (the ratio
between the number of minority and majority examples).\\ Let $X$ be
an attribute and $Y$ be the response class. Here $Y^+$ denotes
majority class, $Y^-$ denotes minority class and $n$ is the total
number of instances. Also, let $X^\geq\longrightarrow Y^+$ and
$X^<\longrightarrow Y^-$ be two rules generated by CT. Table 1 shows
the number of instances based on the rules created using CT.
\begin{table}[H]
\small \centering \caption{An example of notions of classification
rules}
    \begin{tabular}{cccc}
        \hline
        class and attribute     & $X^\geq$ & $X^<$  & sum of instances \\ \hline
        $Y^+$                   & $a$      & $b$    & $a+b$            \\
        $Y^-$                   & $c$      & $d$    & $c+d$            \\ \hline
        sum of attributes       & $a+c$    & $b+d$  & $n$              \\
        \hline
    \end{tabular}

\end{table}
In the case of imbalanced dataset the majority class is always much
larger than the size of the minority class, thus we will always have
$a + b > > c + d$. It is clear that the generation of rules based on
confidence in CT is biased towards majority class. Various measures
like information gain (IG), gini index (GI) and misclassification
impurity (MI) are expressed as a function of confidence, used to
decide which variable to split in the important feature selection
stage \citep{flach2003geometry}. From Table 1, we can define
Confidence$(X^\geq\longrightarrow Y^+) = \frac{a}{a+c}$. Let us
consider a binary classification problem with the label set $\Omega$
= \{$\omega_{1},\omega_{2}$\} and let $P(j /t)$ be the probability
for class $\omega_{j}$ at a certain node $t$ of the classification
tree, where, $j = 1, 2$ for binary classification problems. These
probabilities can be estimated as the proportion of points from the
respective class within the data set that reached the node $t$.
Using Table 1, we compute the following:
\begin{eqnarray}
P(Y^{+}/X^\geq)= \frac{a}{a+c} =
\mbox{Confidence}(X^\geq\longrightarrow Y^+)
\end{eqnarray}
For an imbalanced dataset, $Y^+$ will occur more frequently with
$X^\geq$ \& $X^<$ than to $Y^-$. So the concept of confidence is a
fatal error in an imbalanced classification problem where minority
class is of more interest and data is biased towards the majority
class. In binary classification, information gain for splitting a
node $t$ is defined as:
\begin{eqnarray}
\mbox{IG}=
\mbox{Entropy}(t)-\sum_{i=1,2}\frac{n_{i}}{n}\mbox{Entropy}(i)
\end{eqnarray}
where $i$ represents one of the sub-nodes after splitting (assuming
we have two sub-nodes only), $n_{i}$ is the number of instances in
sub-node $i$ and $n$ is the total number of instances. Entropy at
node $t$ is defined as:
\begin{eqnarray}
\mbox{Entropy}(t)= -\sum_{j=1,2}P(j /t)log\big(P(j /t)\big)
\end{eqnarray}
The objective of classification using CT is to maximize IG which
reduces to (assuming the training set is fixed and so the first term
in equation (2) is fixed as well):
\begin{eqnarray}
\mbox{Maximize}\bigg\{-\sum_{i=1,2}\frac{n_{i}}{n}\mbox{Entropy}(i)\bigg\}
\end{eqnarray}
Using Table 1 and equation (3); the maximization problem in equation
(4) reduces to:
\begin{multline}
\mbox{Maximize}\bigg\{\frac{n_{1}}{n}\Big[P(Y^{+}/X^\geq)log\big(P(Y^{+}/X^\geq)\big)+P(Y^{-}/X^\geq)log\big(P(Y^{-}/X^\geq)\big)]\\
+\frac{n_{2}}{n}[P(Y^{+}/X^<)log\big(P(Y^{+}/X^<)\big)+P(Y^{-}/X^<)log\big(P(Y^{-}/X^<)\big)\Big]\bigg\}
\end{multline}
The task of selecting the ``best" set of features for node $i$ are
carried out by picking up the feature with maximum IG. As
$P(Y^{+}/X^\geq)>>P(Y^{-}/X^\geq)$, we face a problem while
maximizing (5). We can conclude from the above discussion that
feature selection in CT based on the impurity measures is biased
towards majority class.\\

This imbalanced problem can be looked upon from the perspective of
performance evaluation metrics as well. Standard notations of a
confusion matrix are given in Table 2. From Table 2 and using
equation (1), we can write
$P(Y^{+}/X^\geq)=\frac{\mbox{TP}}{\mbox{TP+FP}}>P(Y^{-}/X^\geq)=\frac{\mbox{FP}}{\mbox{TP+FP}}$,
etc. Since the misclassification rate in the minority class is
higher than the misclassification rate in the majority class, the
confusion matrix based on different classification algorithms will
have a fatal error. As a consequence, prediction models with
imbalanced data will lead to a high false negative rate.

\begin{table}[H]
\small \centering \caption{Confusion matrix for binary
classification problem}
    \begin{tabular}{ccc}
        \hline
        All instances           & predicted majority class & predicted minority class \\
        \hline
        actual majority class   & True Positive (TP)       & False Negative(FN)       \\
        actual minority class   & False Positive (FP)      & True Negative (TN)       \\
        \hline
    \end{tabular}

\end{table}

\section{Solution Methodology}  \label{Proposed Methodology for Handling Imbalanced Dataset}

\subsection{\textbf{An Overview on HDDT}\\}

One way of handling an imbalanced dataset is to take recourse to
sampling techniques during the preparation of the data set for
further analysis. However, a significant disadvantage of these
techniques is that in the process of sampling we either lose a lot
of information in the form of losing the real-life data. HDDT uses
HD as the splitting criterion to build a decision tree, has the
property of skew insensitivity \citep{cieslak2008learning}. For the
application of HD as a decision tree criterion, the final
formulation can be written as follows:
\begin{equation}
HD =
d_{H}(X_{+},X_{-})=\sqrt{\sum_{j=1}^{k}\bigg(\frac{|X_{+j}|}{|X_{+}|}-\frac{|X_{-j}|}{|X_{-}|}\bigg)^{2}},
\end{equation}
where $|X_{+}|$ indicates the number of examples that belong to the
majority class in training set and $|X_{+j}|$ is the subset of
training set with the majority class and the value $j$ for the
feature $X$. Similar explanation can be written for $|X_{-}|$ and
$|X_{-j}|$ but for the minority class. Here $k$ is the number of
partitions of the feature space $X$. The bigger the value of HD, the
better is the discrimination between the features. A feature is
selected that carries the minimal affinity between the classes.
Since equation (6) is not influenced by prior probability, it is
insensitive to the class distribution. Based on the experimental
results, Chawla \citep{cieslak2008learning} concluded that unpruned
HDDT is recommended for dealing with imbalanced problems as a better
alternative to sampling approaches.

%Let $(\Theta, \lambda)$ denote a measurable space. For any binary
%classification problem, let us suppose that $P$ and $Q$ be two
%continuous distributions with respect to the parameter $\lambda$
%having the densities $p$ and $q$ in a continuous space $\Omega$,
%respectively. Define HD as follows:
%\[
%d_{H}(P,Q)=\sqrt
%{\int_{\Omega}(\sqrt{p}-\sqrt{q})^{2}d\lambda}=\sqrt
%{2\bigg(1-\int_{\Omega}\sqrt{pq}d\lambda\bigg)}
%\]
%where $\int_{\Omega}\sqrt{pq}d\lambda$ is the Hellinger integral. It
%is noted that HD doesn't depend on the choice of the parameter
%$\lambda$. Given a countable space $\Phi$, HD can also be written as
%follows:
%\[
%d_{H}(P,Q)=\sqrt{\sum_{\phi\in\Phi}\bigg(\sqrt{P(\phi)}-\sqrt{Q(\phi)}\bigg)^{2}}
%\]

\subsection{\textbf{An Overview on RBFN}\\}

RBFN is a family of ANNs, consists of only a single hidden layer and
uses a nonlinear function called radial basis function as an
activation function, unlike MLP. Figure 1 gives an example of RBFN
with one hidden layer. RBFN is presented as a three-layered
feed-forward structure where input layer distributes inputs to the
hidden layer which contains neurons with nonlinear activation
function \citep{gyorfi2006distribution}. Since gaussian functions
are most frequently used in this layer, we use the gaussian kernel
in this paper:

\[
\phi_i(x_i)=\phi \big(\parallel x_i - c_{i} \parallel ; \sigma_{i}
\big)= exp\bigg(- \frac{{\parallel x_i - c_{i} \parallel}^{2}}{2
\sigma_{i}^2} \bigg)
\]
where $x$ is an input vector, $\phi_i$ is the output of $i^{th}$
hidden neuron in the hidden layer with centers $c_i$ and $\sigma_i$
as the standard deviation. Finally, the output layer can be written
as a weighted sum of hidden layer outputs:

\[
f(x_i)=\sum_{j=1}^{k} w_{j} \; \phi \big(\parallel x_i - c_{j}
\parallel \big)+ w_{0},
\]
where $w_{j}$ is the weight of the link from $j^{th}$ hidden neuron
to the $i^{th}$ output neuron. An interesting property of RBFN which
distinguishes it from other types of ANNs is that here the center
vector can be selected as cluster centers of the input data. For
practical use, the number of clusters is usually chosen to be much
smaller than the number of data points resulting in RBFN of less
complexity than other types of ANNs.

\begin{figure}[t]
\centering
%\subfloat[]
%\subfig{
%\includegraphics[height=1cm,angle=0,width=.6\linewidth,]{mod_realworld.eps}%
  %\epsfig{file=amazon_sim_final.ps, height=2.5in, width=3in}
%\includegraphics[width=0.6\textwidth,natwidth=1250,natheight=1250]{mod_real_world.PNG}
%\includegraphics[scale=0.48]{2.eps}
\includegraphics[scale=0.40]{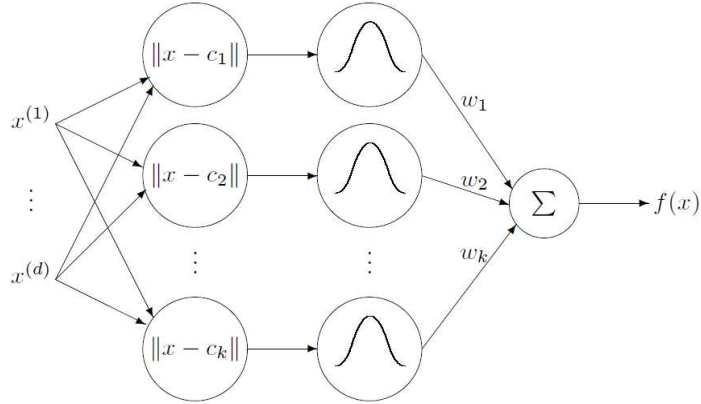}
\caption{An example of Radial basis network with one hidden layer.}
\end{figure}

\subsection{\textbf{Proposed Superensemble Classifier}\\}

The motivation behind designing a superensemble classifier for
imbalanced data sets is: (1) one would like to work with the
original dataset without taking recourse to sampling; (2) we would
like to ignore the drawbacks of single classifiers (HDDT and RBFN)
and harnessing their positiveness; (3) high prediction accuracy.
Here we are going to discuss our proposed approach which utilizes
the power of HDDT as well as the superiority of RBF networks.

In the proposed model, we first split the feature space into areas
by HDDT algorithm (discussed in Section 3.1). Based on feature
rankings provided by HDDT, a set of important features are chosen
and extracted from the training dataset. We then build the RBFN
model using the important variables obtained through HDDT algorithm
along with the prediction results obtained from HDDT as another
input information in the input layer of the network. The
effectiveness of the proposed classifier lies in the selection of
important features and use of prediction results of HDDT followed by
the application of the RBFN model. The inclusion of HDDT output as
an additional input feature not only improves the model accuracy but
also increases class separability. The proposed superensemble
classifier can handle imbalance through the implementation of HDDT
in selecting features as well as the incorporation of its predicted
classes tied up with one hidden layered RBFN model. This
hybridization improves the performances of single classifiers as
well as reduces the biases and variances of HDDT and RBFN. The
informal workflow of our proposed model is as follows:

\begin{itemize}
\item Apply HDDT algorithm to rank the feature, extracts
important features and find the splits between different adjacent
values of the feature using HD values (equation (6)).
\item Choose the features that have maximum HD value. We further grow
unpruned HDDT and record its outputs. HDDT has an in-built important
feature selection mechanism and it takes into account the imbalanced
nature of the dataset.
\item Export important input variables along with additional feature
(prediction result of HDDT algorithm) to the RBFN model and a neural
network is generated.
\item RBFN model uses a gaussian kernel as an activation function, and
parameter optimization is done using a gradient descent algorithm
(to be discussed in Section 4.2). Finally, we record the final
predicted class levels.
\end{itemize}

\begin{figure}[t]
\centering
%\subfloat[]
%\subfig{
%\includegraphics[height=1cm,angle=0,width=.6\linewidth,]{mod_realworld.eps}%
  %\epsfig{file=amazon_sim_final.ps, height=2.5in, width=3in}
%\includegraphics[width=0.6\textwidth,natwidth=1250,natheight=1250]{mod_real_world.PNG}
\includegraphics[scale=0.48]{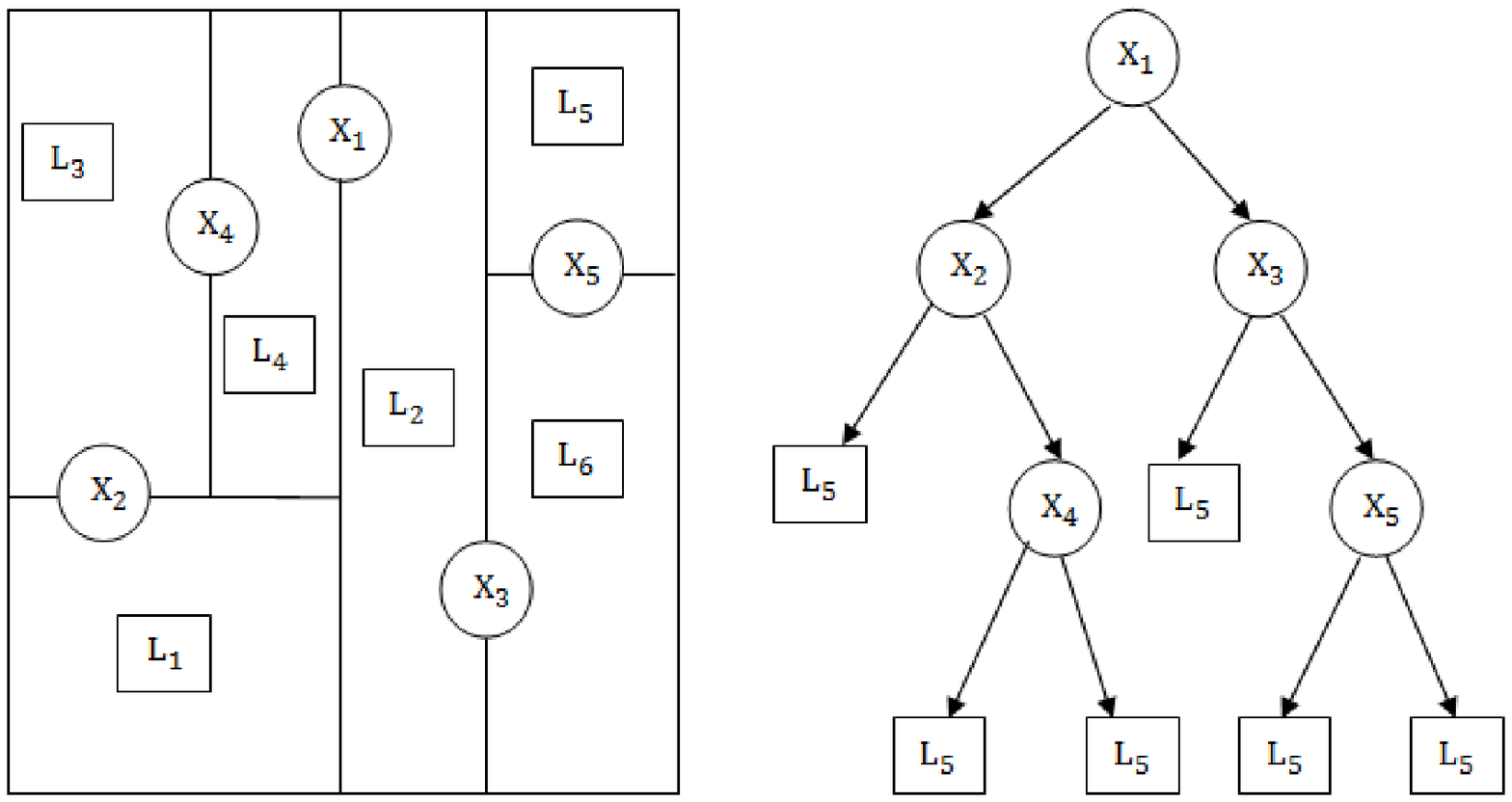}
\includegraphics[scale=0.45]{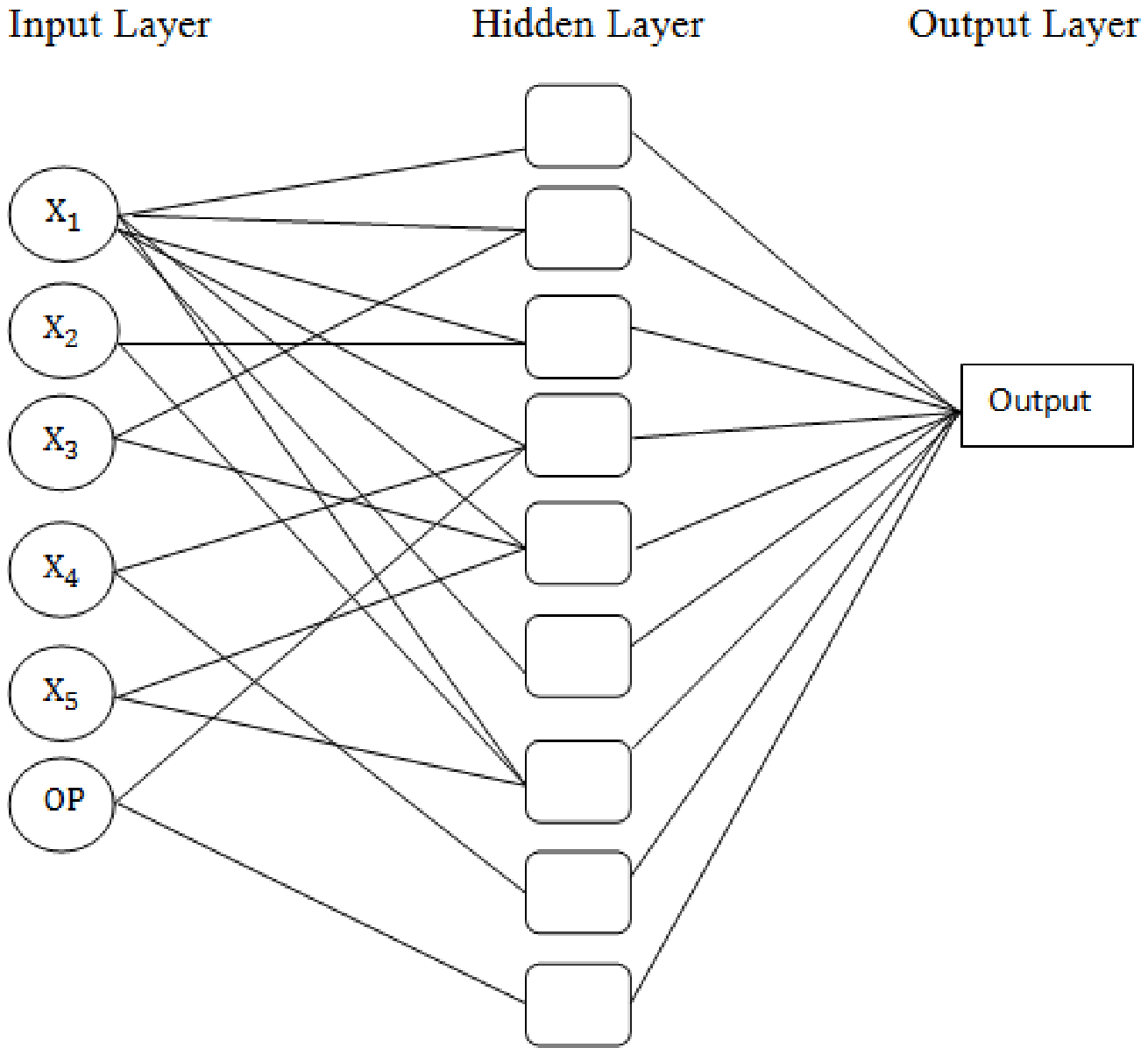}
\caption{An example of superensemble classifier with $X_{i};
i=1,2,3,4,5$ as Important features obtained by HDDT, $L_{i}$ as leaf
nodes and OP as HDDT output. Creation of HDDT (Left) and one hiddn
layered RBFN model (Right).} \label{figmodRWNGT}
\end{figure}

This algorithm is a multi-step problem-solving approach such as
handling imbalanced class distribution, selecting important features
and finally getting an improved prediction. Due to the fact that
HDDT is found to be robust in handling the curse of imbalanced
datasets, thus incorporation of its predicted class levels along
with important features obtained by HDDT as input features in RBFN
will necessarily improve the performance of the single classifiers.
A flowchart of the superensemble classifier is presented in Figure
2.

\section{\textbf{Theoretical Properties of the Proposed Model}}

Our proposed superensemble classifier has the following
architecture: (1) it extracts important features from the feature
space using HDDT algorithm; (2) it builds one hidden layered RBFN
model with the important features extracted using HDDT along with
HDDT outputs as an additional feature. We investigate the
theoretical properties of the proposed model by introducing a set of
regularity conditions for consistency of HDDT followed by the
consistency results of RBFN algorithm. Further, an idea about
parameter optimization in the superensemble model is proposed in
Section 4.2.

\subsection{\textbf{Regularity Conditions for Universal
Consistency}\\}

Let $\underline{X}$ be the space of all possible values of $d$
features and $C$ be the set of all possible binary class labels. We
are given a training sample with $n$ observations
$L=\{(X_{1},C_{1}), (X_{2},C_{2}),\\...,(X_{n},C_{n})\}$, where
$X_{i}=(X_{i1},X_{i2},...,X_{id}) \in \underline{X}$ and $C_{i} \in
C$. Also let $\Omega=\{\omega_{1},\omega_{2},...,\omega_{k_n}\}$ be
a partition of the feature space $\underline{X}$. We denote
$\widetilde{\Omega}$ as one such partition of $\Omega$. Define
$L_{\omega_{i}}=\{(X_{i},C_{i})\in L: X_{i}\in \omega_{i}, C_{i}\in
C\}$ as the subset of $L$ induced by $\omega_{i}$ and let
$L_{\widetilde{\Omega}}$ denote the partition of $L$ induced by
$\widetilde{\Omega}$. The value of HD obtained using Eqn. (6) is
used to partition $\underline{X}$ into a set $\widetilde{\Omega}$ of
nodes and then we can construct a classification function on
$\widetilde{\Omega}$. There exists a partitioning classification
function $p:\widetilde{\Omega} \rightarrow C$ such that $p$ is
constant on every node of $\widetilde{\Omega}$. Now, let us define
$\widehat{L}$ to be the space of all learning samples and
$\mathbb{D}$ be the space of all partitioning classification
function, then $\Phi:\widehat{L} \rightarrow \mathbb{D}$ such that
$\Phi(L)=(\psi \circ \phi)(L)$, where $\phi$ maps $L$ to some
induced partition $(L)_{\widetilde{\Omega}}$ and $\psi$ is an
assigning rule which maps $(L)_{\widetilde{\Omega}}$ to $p$ on the
partition $\widetilde{\Omega}$. The most basic reasonable assigning
rule $\psi$ is the plurality rule
$\psi_{pl}(L_{\widetilde{\Omega}})=p$ such that if $x \in
\omega_{i}$, then
\[
p(\underline{x})=\arg \max_{c \in C}|L_{c,\omega_{i}}|.
\]
The plurality rule is used to classify each new point in
$\omega_{i}$ as belonging to the class (either 0 or 1 in this case)
most common in $L_{\omega_{i}}$. This rule is very important in
proving risk consistency of the HDDT algorithm. Binary split
function partitions a parent node $\omega_{i} \subseteq
\underline{X}$ into a non-empty child nodes $\omega_{1}$ and
$\omega_{2}$, called left child and right child node respectively.
We use HD measure (see in Section 3.1) as goodness of split
criterion and for any parent node $\omega_{i}$, the goodness of
split criterion ranks the split function and build the tree. The
stopping rule in HDDT is chosen as the minimum number of split in
the posterior sample, called minsplit function ($r(\omega_{i})$). If
$r(\omega_{i})$ $\geq \alpha$, then $\omega_{i}$ will split into two
child nodes and if $r(\omega_{i})$ $< \alpha$, then $\omega_{i}$ is
a leaf node and no more split is required. Here $\alpha$ is
determined by the user, though for practice it is usually taken as
10\% of the training sample size. \\

Thus HDDT is a binary tree-based classification and partitioning
scheme $\Phi$, defined as an assignment rule applied to the limit of
a sequence of induced partitions $\phi^{(i)}(L)$, where
$\phi^{(i)}(L)$ is the partition of the training sample $L$ induced
by the partition $(\phi_{i} \circ \phi_{i-1} \circ .... \circ
\phi_{1})(\underline{X})$. For every node $\omega_{i}$ in a
partition $\widetilde{\Omega}$ such that $r(\omega_{i})$ $\geq
\alpha$, then the function $\phi(\widetilde{\Omega})$ splits each
node into two child nodes using the binary split in the question set
as determined by $\mathscr{G}$. For other nodes $\omega_{i} \in
\widetilde{\Omega}$ such that $r(\omega_{i})$ $< \alpha$, then
$\phi(\widetilde{\Omega})$ leaves $\omega_{i}$ unchanged. For a
partition $\widetilde{\Omega}$ of $X$, let $\widetilde{\Omega}[x \in
X]=\{ \omega \in \widetilde{\Omega}: x \in \omega \}$ be the node
$\omega$ in $\widetilde{\Omega}$ which contains $x$.\\
Mathematically, we can write
\begin{eqnarray*}
\Phi(L)=(\psi \circ \lim_{i\rightarrow \infty}\phi^{(i)})(L) \quad
\mbox{\big(where, $\phi^{(i)}(L)=L_{(\phi_{i} \circ \phi_{i-1} \circ
.... \circ \phi_{1})(\underline{X})}$\big)}.\\
\end{eqnarray*}

The conditions for risk consistency of HDDT are given in Theorem 1.

\begin{theorem}
Suppose $(\underline{X},\underline{Y})$ be a random vector in
$\mathbb{R}^{d}\times\ C$ and $L$ be the training set consisting of
$n$ outcomes of $(\underline{X},\underline{Y})$. Let $\Phi$ be a
HDDT scheme such that
\[
\Phi(L)=(\psi_{pl}\circ \lim_{i\to\infty}\phi^{(i)})(L) \quad \quad
\mbox{(where, $\psi_{pl}$ is the plurality rule).}
\]
Also suppose that all the split function in HDDT in the question set
associated with $\Phi$ are axis-parallel splits and satisfies
shrinking cell condition. If for every n and $w_{i} \in
\tilde{\Omega}_{n}$, the induced subset $L_{w_{i}}$ has cardinality
$\geq k_n$, where $\frac{k_n}{log(n))}\rightarrow \infty$, then
$\Phi$
is risk consistent.\\
\end{theorem}

\begin{proof}
The proof of the theorem is very similar to the proof of the
consistency of the classification tree
\citep{chakraborty2018nonparametric} and histogram based
partitioning and classification schemes
\citep{lugosi1996consistency}. The conditions for consistency
require each cell of every partition belongs to a fixed
Vapnik-Chervonenkis class of sets and also every cell must contain
$k_n$ points satisfying $\frac{k_n}{log(n))}\rightarrow \infty$ as
sample size approaches to infinity. For detailed proof, the reader
may refer to
\citep{chakraborty2018nonparametric, lugosi1996consistency}.\\
\end{proof}

\begin{remark}
It should be noted that the choice of important features based on
HDDT is a greedy algorithm and the optimality of local choices of
the best feature for a node doesn't guarantee that the constructed
tree will be globally optimal \citep{kuncheva2004combining}.
\end{remark}

Further, we build the RBFN model with HDDT extracted features and
$OP$ as another input feature in the RBFN model. The dimension of
the input layer in the RBFN model, denoted by $d_{m}(\leq{d})$, is
the number of important features obtained by HDDT + 1. RBFN model
consists of strictly one hidden layer, and due to this, the
superensemble model is easily interpretable and fast in
implementation. Our next objective is to discuss the sufficient
condition for universal consistency of RBFN model. After
incorporating HDDT output in the feature space, we have $n$ training
sequence $\xi_{n}=\{(Z_{1},Y_{1}),...,(Z_{n},Y_{n})\}$ of $n$ i.i.d
copies of $(\underline{Z},\underline{Y})$ taking values from
$\mathbb{R}^{d_{m}}\times C$, a classification rule realized by a
one-hidden layered neural network is chosen to minimize the
empirical $L_{1}$ risk. Define the $L_{1}$ error of a function $\psi
: \mathbb{R}^{d_{m}}\rightarrow \{0,1\}$ by $J(\psi)=E\{ |\psi(Z)-Y|
\}$.\\\\ Consider a RBF network with one hidden layer and at most
$k$ nodes for a fixed gaussian function $\phi$ (as defined in
Section 3.2), is given by the equation:

\[
f(z_i)=\sum_{j=1}^{k} w_{j} \; \phi \big(\parallel z_i - c_{j}
\parallel \big)+ w_{0},
\]
where $w_0, w_1,...,w_k \in [-b,b] \; (b>0)$ and $c_1, c_2,...,c_k
\in \mathbb{R}^{d_{m}}$. The weights $w_j$ and $c_j$ are parameters
of the RBF network and $\phi$ is the radial basis function. Choosing
gaussian basis function as radial basis function which is a
decreasing function such that $\phi(z) \rightarrow 0 \; as \; x
\rightarrow \infty$. The next Theorem due to krzyzak et al.
\citep{krzyzak1996nonparametric} gives regularity conditions for the
universal consistency of RBFN model. But for the sake of
completeness, we are rewriting Theorem 2 of
\citep{krzyzak1996nonparametric} in our context.

\begin{theorem}
If $k \rightarrow \infty, \; k=o\big(\frac{n}{logn}\big) \;
\mbox{as} \; n \rightarrow \infty $ in the RBFN model having
gaussian radial basis function $\phi$, and in which all the
parameters are chosen by empirical risk minimization, then RBFN
model is said to be universally consistent
for all distribution of $(\underline{Z}, \underline{Y})$.\\
\end{theorem}

\begin{remark}
We can conclude that if the superensemble classifier satisfies the
regularity conditions as stated in Theorem 1 and Theorem 2, then the
algorithm will be universally consistent. This is a fundamental
property of any classifier for its robustness and general use. But
computationally choosing parameters of RBFN by minimizing empirical
$L_1$ risk will be very costly. In practice, parameters of the RBFN
model are learned by a gradient descent algorithm, discussed in the
next subsection.
\end{remark}

\subsection{\textbf{Optimization of Model Parameters}\\}

Here we will discuss the tuning parameters of the proposed
superensemble classifier. In the first stage of the pipeline model,
`minsplit' function (see Section 4.1) to be chosen as $10\%$ of the
training dataset is recommended as the stopping rule in HDDT
algorithm. Further, we grow unpruned HDDT and use HDDT suggested
features and HDDT output as an additional feature in the input
feature space of RBFN. RBF network is designed using a linear
combination of gaussian basis functions. Therefore we need to use
some optimization algorithm for empirical error (to be denoted by E
in the rest of the paper) minimization on $\xi_n$
\citep{poggio1990regularization}. Three important parameters to be
optimized while training RBF network are: centers $(c_i)$, standard
deviation $(\sigma_i)$ and weights $(w_j)$ of each neuron. We use a
gradient descent algorithm over E to perform the optimization task
in RBFN model \citep{karayiannis1999reformulated}, as follows:

\[
\Delta c_i = - \rho_c \nabla_{c_{i}}E,
\]

\[
\Delta \sigma_i = - \rho_\sigma \frac{\partial E}{\partial
\sigma_i},
\]

\[
\Delta w_j = - \rho_w \frac{\partial E}{\partial w_j},
\]
where $\rho_c, \rho_\sigma \; \mbox{and} \; \rho_w$ are small
positive constants. By this way, the parameters of the gaussian
basis function will be optimized. Generalization error is estimated
by cross-validation method, and the optimum value of $k$ can be
found by trial and error.

\section{Computational Experiments}  \label{experimental_eval}

In this section, we describe the datasets in brief and also discuss
performance evaluation metric. Subsequently, we are going to report
the experimental results and compare our proposed model with other
state-of-the-art classifiers.

\subsection{\textbf{Data Description}}

The proposed model is evaluated using five publicly available
datasets from a wide variety of application areas such as
management, business, and medicine, available at UCI Machine
Learning repository \citep{asuncion2007uci} and a previous study
\citep{chakraborty2018novel}. Breast cancer dataset consists of 9
discrete features whereas pima diabetes dataset has 8 continuous
features in its feature space. German credit card dataset (also
popularly know as Statlog dataset) consists of 13 qualitative
features and 7 numerical features. In this dataset, entries
represent persons who take credit by a bank, and each person is
classified as good or bad credit risks according to the set of
attributes. Page blocks database has numeric attributes, contain
blocks of the page layout of a document that has been detected by a
segmentation process. Indian business school dataset contains 10
continuous and 7 categorical variable on the characteristics of
students admitting in a business school and the response variable
denotes whether the student will be placed or not at the end of the
curriculum. To measure the level of imbalance of these datasets, we
compute the coefficient of variation (CV) which is the proportion of
the deviation in the observed number of samples for each class
versus the expected number of examples in each class
\citep{wu2010cog}. We have chosen thee datasets with a CV more than
equal to $0.30-$ a class ratio of $2:1$ on a binary dataset as
imbalanced data. Table 3 gives an overview of these data sets.

\begin{table}[H]
\footnotesize \centering \caption{Characteristics of the data sets
used in experimental evaluation}
    \begin{tabular}{ccccccc}
        \hline
        Dataset                  & Classes    & Objects & Number of      & Number of          & Number of          & CV    \\
                                 &            & $(n)$   & feature $(p)$  & $(+)$ve instances  & $(-)$ve instances  &       \\ \hline
        breast cancer            & 2          & 286     & 9              & 201                & 85                 & 0.41  \\
        german credit card       & 2          & 1000    & 20             & 700                & 300                & 0.40  \\
        indian business school   & 2          & 480     & 17             & 400                & 80                 & 0.56  \\
        page blocks              & 2          & 5473    & 10             & 4913               & 560                & 0.80  \\
        pima diabetes            & 2          & 768     & 8              & 500                & 268                & 0.30  \\
        \hline
    \end{tabular}

\end{table}

\subsection{\textbf{Performance Metrics}}

The performance evaluation measure used in our experimental analysis
is based on the confusion matrix in Table 2. Area under the receiver
operating characteristic curve (AUC) is a popular metric for
evaluating performances of imbalanced datasets and higher the value
of AUC, the better the
classifier is.\\\\
AUC = $\frac{\mbox{Sensitivity}+\mbox{Specificity}}{2}$; where,
Sensitivity = $\frac{TP}{TP+FN}$; Specificity = $\frac{TN}{FP+TN}$.

\subsection{\textbf{Results and Comparisons}}

In order to show the impact of the proposed superensemble
classifier, it is applied to the high-dimensional small or
medium-sized datasets from various applied areas. These are such
types of datasets in which not only classification is the task but
also feature selection plays a vital role as well. To start, we
first shuffled the observations in each of the five different
datasets randomly and split them into training, validation and test
datasets in a ratio of $50 : 25 : 25$. We have repeated each of the
experiments 5 times with different randomly assigned training,
validation and test sets. Further, we compare our proposed
classifier mostly with ``imbalanced data-oriented" classifiers as
baseline comparisons. Even we apply different sampling approaches
over traditional classifiers and evaluate AUC values to see the
competitiveness of the proposed model. All the classifiers are
implemented in the R statistical package, and sampling techniques
are applied in python toolbox on a PC with 8GB memory.

\subsubsection{Baseline Comparisons\\}

We start the experimental analysis by implementing CT algorithm to
five publicly available imbalanced datasets. Tree-based CT model is
trained using \textit{rpart} package implementation in R. CT uses
gini index, and $C_p$ has been used for selection of variables to
enter and leave the tree structure. Further, an ensemble of trees,
random forests (RF), was implemented using \textit{randomForest}
package in R. We reports their prediction performances in Table 4.
Another  simple nonparametric algorithm, k-nearest neighbor (RF) is
applied to the datasets using \textit{class} implementation in R. To
implement neural nets, we first standardize the datasets and run the
ANN model. And we used ``logsig" transfer function to bring back the
original form at the end of modeling. We implemented the ANN model
with different combinations of hidden layers without employing any
other feature selection algorithm using \textit{neuralnet} package.
Since the datasets are small or medium sample sized, thus going
beyond 2 hidden layered (2HL) neural net will overfit the datasets.
For one hidden layer (1HL) ANN model, the number of hidden neurons
are chosen as 2/3 the size of the input layer, plus the size of the
output layer. But for 2HL ANN model number of neurons in 1st HL are
chosen as 2/3 the size of the input layer and the number of neurons
in 2nd HL are chosen as 1/3 the size of the input layer. RBFN is a
particular class of ANN which uses radial basis kernel for nonlinear
classification. Using \textit{RSNNS} package, we applied the RBFN
model with gaussian kernel function, and the maximum number of
iterations in all these NN implementations are chosen as 100.
Execution time for RBFN model is lesser than ANN but higher than
tree-based models.

We now implemented HDDT algorithm by using R Package `CORElearn' for
learning from imbalanced datasets. HDDT usually achieved higher
accuracy than CT and RF. This indicates that ``imbalanced
data-oriented" classifiers perform better than the conventional
supervised classifiers designed for general purposes. Further, we
implemented HDRF, CCPDT which are among other imbalanced
data-oriented algorithms. Finally, we applied our proposed
superensemble classifier which is a pipeline model. In the first
stage, we select important features using HDDT and record its
classification outputs. In the next step, we design a neural network
with the above mentioned important features along with HDDT output
as an additional feature vector. While RBFN implementation, we used
gaussian kernel and used gradient descent algorithm for parameter
optimization (see details in Section 4.2). HDDT reduces the
dimension of the feature set and passes HDDT outputs (taking into
account the curse of dimensionality) to the RBFN model, and RBFN
improves the predictions. We reported the performance of various
classifiers in terms of AUC value in Table 4. It is clear from Table
4 that our proposed methodology achieved better performance than
other state-of-the-art for most of the datasets used in this study
other than page blocks data.

\begin{table}[H]
\footnotesize \centering \caption{AUC results (and their standard
deviation) of classification algorithms over original imbalanced
test datasets}
    \begin{tabular}{cccccc}
        \hline
        Classifiers                         & breast                 & German credit         & Indian business   & page               & pima            \\
                                            & cancer                 & card                  & school            & blocks             & diabetes        \\
                                            \hline
        CT                                  & 0.603 (0.04)           & 0.665 (0.03)          & 0.810 (0.04)      & 0.950 (0.00)       & 0.724 (0.02)   \\
        RF                                  & 0.690 (0.06)           & 0.725 (0.03)          & 0.850 (0.04)      & 0.964 (0.00)       & 0.747 (0.04)   \\
        k-NN                                & 0.651 (0.03)           & 0.727 (0.01)          & 0.750 (0.03)      & 0.902 (0.02)       & 0.730 (0.05)   \\
        RBFN                                & 0.652 (0.06)           & 0.723 (0.04)          & 0.884 (0.05)      & 0.935 (0.01)       & 0.725 (0.04)   \\
        HDDT                                & 0.625 (0.04)           & 0.738 (0.04)          & 0.933 (0.02)      & 0.974 (0.00)       & 0.760 (0.02)   \\
        HDRF                                & 0.636 (0.04)           & 0.742 (0.03)          & 0.939 (0.02)      & \textbf{0.988} (0.00) & 0.760 (0.03)   \\
        CCPDT                               & 0.618 (0.05)           & 0.712 (0.05)          & 0.912 (0.03)      & 0.971 (0.00)       & 0.753 (0.01)   \\
        ANN (with 1HL)                      & 0.585 (0.03)           & 0.700 (0.03)          & 0.768 (0.05)      & 0.918 (0.02)       & 0.649 (0.03)   \\
        ANN (with 2HL)                      & 0.621 (0.02)           & 0.715 (0.02)          & 0.820 (0.04)      & 0.925 (0.01)       & 0.710 (0.03)   \\
        Superensemble Classifier            & \textbf{0.720} (0.06)  & \textbf{0.798} (0.04) & \textbf{0.964} (0.01) & 0.985 (0.00)   & \textbf{0.789} (0.05)\\
        \hline
    \end{tabular}
\end{table}

\subsubsection{Experiments with Sampling Techniques\\}

\begin{table}[H]
\scriptsize \centering \caption{AUC value (and their standard
deviation) for balanced datasets (using SMOTE and SMOTE+ENN) on
different classifiers}
    \begin{tabular}{cccccccc}
        \hline
        Data                                 & Sampling   & kNN          & CT            & RF            & ANN          & ANN          & RBFN          \\
                                             & Techniques &              &               &               & (with 1HL)   & (with 2HL)   &               \\
                                             \hline
        \multirow{2}{*}{breast cancer }      & SMOTE      & 0.700 (0.02) & 0.665 (0.05)  & \textbf{0.722} (0.01)  & 0.605 (0.07) & 0.680 (0.05) & 0.704 (0.04)  \\
                                             & SMOTE+ENN  & 0.685 (0.03) & 0.650 (0.03(  & 0.708 (0.01)  & 0.600 (0.06) & 0.652 (0.06) & 0.700 (0.03)  \\
                                             \hline
        \multirow{2}{*}{german credit card}  & SMOTE      & 0.758 (0.04) & 0.745 (0.02)  & 0.762 (0.01)  & 0.740 (0.03) & 0.735 (0.02) & 0.764 (0.00)  \\
                                             & SMOTE+ENN  & 0.760 (0.03) & \textbf{0.778} (0.02)  & 0.770 (0.02)  & 0.750 (0.02) & 0.720 (0.02) & 0.765 (0.00)  \\
                                             \hline
        \multirow{2}{*}{indian business school}& SMOTE    & 0.783 (0.01) & 0.845 (0.02)  & 0.859 (0.01)  & 0.765 (0.01) & 0.798 (0.03) & 0.905 (0.01)  \\
                                             & SMOTE+ENN  & 0.801 (0.01) & 0.850 (0.01)  & 0.875 (0.00)  & 0.798 (0.02) & 0.807 (0.03) & \textbf{0.914} (0.00)  \\
                                             \hline
        \multirow{2}{*}{page blocks}         & SMOTE      & 0.927 (0.01) & 0.965 (0.00)  & \textbf{0.967} (0.01)  & 0.933 (0.01)    & 0.942 (0.02) & 0.954 (0.00) \\
                                             & SMOTE+ENN  & 0.935 (0.01) & 0.952 (0.00)  & 0.966 (0.01)           & 0.925 (0.01)    & 0.937 (0.01) & 0.949 (0.01) \\
                                             \hline
        \multirow{2}{*}{pima diabetes}       & SMOTE      & 0.770 (0.05)          & 0.758 (0.02) & 0.753 (0.03)   & 0.698 (0.05)    & 0.719 (0.02) & 0.745 (0.03) \\
                                             & SMOTE+ENN  & \textbf{0.788} (0.04) & 0.760 (0.01) & 0.761 (0.02)   & 0.712 (0.04)    & 0.725 (0.02) & 0.748 (0.02) \\
                                             \hline
    \end{tabular}
\end{table}

The application of data preprocessing steps to balance the class
distributions is found useful to solve the curse of dimensionality.
It has an advantage that these approaches are independent of the
classifiers used \citep{batista2004study}. But the major
disadvantage is that it changes the original dataset. In this
subsection, we try to find out the effectiveness of sampling methods
on the imbalanced datasets. ``Imbalanced-learn" is a python toolbox,
used to tackle the curse of imbalanced datasets, provides
application of a wide range of available sampling methods
\citep{lemaitre2017imbalanced}. SMOTE \citep{chawla2002smote} and
SMOTE+ENN \citep{batista2004study} are among the most popular
oversampling and combination of oversampling and undersampling
methods, respectively. We implement these two sampling techniques
using \textit{``imbalanced-learn"} open-source toolbox in
python\footnote{https://github.com/scikit-learn-contrib/imbalanced-learn}
with the default parameters available in the toolbox. These methods
provide us balanced datasets with equal class distributions. On the
balanced datasets, we implement traditional models like k-NN, CT,
RF, ANN (1HL), ANN (2HL), RBFN models. Results based on the
application of different single traditional classifiers on two
different sampling techniques are reported in Table 5.

\begin{remark}
From Table 4 and Table 5, we can conclude that our proposed model
performs reasonably well as compared to baseline models as well as
single classifiers oversampling methods. We have highlighted the
highest AUC value in both the tables with dark black for all the
datasets. It is clear from computational experiments that our model
stands as very much competitive with current state-of-the-art
models.
\end{remark}

\section{Conclusions and Discussions} \label{conclusion and Discussion}

In this paper, we proposed a novel distribution-free superensemble
classifier which is a hybridization of HDDT and RBFN model for
improving predictions in binary imbalanced classification problems.
Because allocating half of the training examples to the minority
class doesn't provide the optimal solution in most of the real-life
problems \citep{weiss2003learning}. It is important to note that
``imbalanced data-oriented" algorithms perform well on the original
imbalanced datasets \citep{cieslak2012hellinger}. If one would like
to work with the original data without taking recourse to sampling,
our proposed methodology will be quite handy. Our proposed model
takes into account data imbalance and is useful for feature
selection cum classification problems having small or medium-sized
datasets. Through computational experiments, we have shown our
proposed methodology performed better as compared to the other
state-of-the-art. The model also has the desired statistical
properties like universal consistency, less tuning parameters and
achieves higher accuracy. We thereby conclude that for the
imbalanced datasets it is sufficient to use superensemble classifier
without taking recourse to sampling or any other ``imbalanced
data-oriented" single classifiers. The usefulness and effectiveness
of the model lie in its robustness and easy interpretability as
compared to complex ``black-box-like" models. The proposed
classifier will perform significantly well in imbalanced frameworks
where our job is feature selection cum classification. If feature
selection is not a part of the data analysis, like datasets from
controlled/lab experiments (for example, NASA software testing
datasets), our model will perform quite similar to HDDT and RBFN and
very less gain in prediction can be obtained. But no model can have
superior advantages, and this can be justified with \textit{no free
lunch theorem} \citep{wolpert2002supervised}. An immediate extension
of this work is possible for multi-class classification problems in
imbalanced frameworks.

\section*{Acknowledgements}
The authors are grateful to the editors and anonymous referees for
careful reading, constructive comments and insightful suggestions,
which have significantly improved the quality of the paper.

\bibliographystyle{tfcad}
\bibliography{bibliography}

%\bibliographystyle{elsarticle-num}
%\bibliography{bibliography}

\end{document}